\documentclass[conference]{IEEEtran}
\usepackage{times}

\usepackage[numbers]{natbib}
\usepackage{multicol}
\usepackage[bookmarks=true]{hyperref}
\usepackage{graphics} 
\usepackage{epsfig} 
\usepackage{mathrsfs}
\usepackage{times} 
\usepackage{amsmath} 
\usepackage{amssymb}  
\usepackage{algpseudocode}
\usepackage[ruled,vlined,linesnumbered]{algorithm2e}
\usepackage{dblfloatfix}    
\usepackage{optidef}
\usepackage{amsthm}
\usepackage{tensor}
\usepackage{flushend}

\newtheorem{assumption}{Assumption}
\newtheorem{theorem}{Theorem}

\hypersetup{
    colorlinks=true,
    linkcolor=blue,
    filecolor=green,      
    urlcolor=black,
}

\begin{document}
\title{                                                           
SEED: Series Elastic End Effectors in 6D for Visuotactile Tool Use}




%
\author{\authorblockN{H.J. Terry Suh\authorrefmark{1},
Naveen Kuppuswamy\authorrefmark{2},
Tao Pang\authorrefmark{1},
Paul Mitiguy\authorrefmark{2},
Alex Alspach\authorrefmark{2},
Russ Tedrake\authorrefmark{1}\authorrefmark{2}}
\authorblockA{\authorrefmark{1}Massachusetts Institute of Technology,
Cambridge, Massachusetts 02139\\ Email: \{hjsuh, pangtao, russt\}@mit.edu}
\authorblockA{\authorrefmark{2}Toyota Research Institute, Cambridge, Massachusetts 02139\\
Email: \{naveen.kuppuswamy, paul.mitiguy, alex.alspach\}@tri.global}}

\maketitle

\begin{abstract}
We propose the framework of Series Elastic End Effectors in 6D (SEED), which combines a spatially compliant element with visuotactile sensing to grasp and manipulate tools in the wild. Our framework generalizes the benefits of series elasticity to 6-dof, while providing an abstraction of control using visuotactile sensing. We propose an algorithm for relative pose estimation from visuotactile sensing, and a spatial hybrid force-position controller capable of achieving stable force interaction with the environment. We demonstrate the effectiveness of our framework on tools that require regulation of spatial forces. Video link: \href{https://youtu.be/2-YuIfspDrk}{\color{magenta} https://youtu.be/2-YuIfspDrk}
\end{abstract}

\IEEEpeerreviewmaketitle

\section{Introduction}
Many tasks in robot manipulation require handling of general tools in the wild; in the future, we believe that robots will be able to grab any tool and do meaningful control in order to accomplish various tasks and exchange forces with the environment. To manipulate tools skillfully and robustly, we will need end effectors that allow controllable hand-tool interaction in hardware, while having sensing capabilities on this interaction to enable closed-loop feedback.

Parallel-jaw grippers are sufficient for grasping \cite{antipodal}, but quickly meet limitations when it comes to forceful tool use. Even when sensing is given via finger attachments \cite{gelslim}, the hardware often relies on friction to handle the forces that arise in hand-tool interaction, and may lack the ability to resist spatial forces in some of the axes. For example, torques applied perpendicular to the finger surface are hard to resist, but arise in many tool-use scenarios \cite{holladay}. Multi-fingered hands are much more versatile, but existing solutions (e.g. imposing a full-rank grasp matrix on the tool \cite{liandsastry,cutkosky}) also rely on frictional contacts, which can limit the amount of force they can exert.

More fundamentally, the rigidity of most of our hardware requires non-smooth contact forces to be used to resist external forces in tool-use. Such forces can be notoriously hard to control robustly as their behavior changes instantaneously \cite{suh2021bundled}. In the absence of the challenges brought by rigid contacts, custom tool changers that are rigidly attached to the robot have demonstrated impressive capability to achieve finely controlled force interaction with the environment \cite{grinder}. However, such solutions require modifying tools with specialized handles compatible with the tool changer, which limits the robot's ability to use unmodified tools.

To alleviate the difficulties coming from rigidity and the non-smooth behavior that it brings, we ask the following question in this work: can we consider visuotactile hardware as not only a mechanism for sensing, but also as an opportunity to provide \textit{compliance} for control? Indeed, similar ideas have been proposed in Series Elastic Actuators (SEA)s \cite{sea}; by attaching a soft spring in front of the gearbox whose deformation can be measured, SEAs have been successful in achieving smooth and stable force control by turning the problem into that of position control \cite{sea,seathesis,hoganimpedancecontrol}.

\begin{figure}[t]
	\centering\includegraphics[width = 0.5\textwidth]{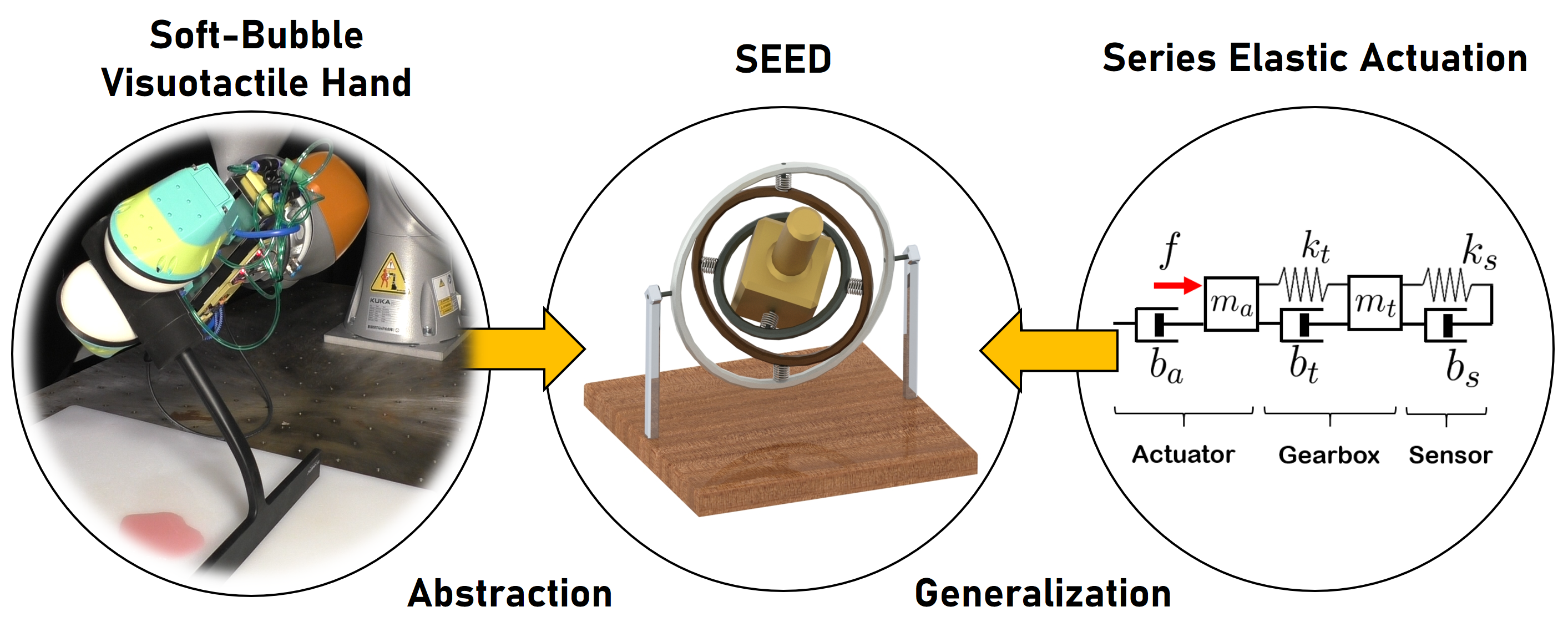}
    \caption{A visual illustration of our framework.}
	\label{fig:banner}
	\vskip -2.0em
\end{figure}

How can we generalize the benefits of SEAs to the setting of grasping and using arbitrary tools? We propose an answer that attaches soft, spatially compliant elements at the end effector right where interaction with the tool occurs. Mechanically, such a solution can be attached to a low-cost, position-controlled robot, while still achieving the benefits of SEAs in the interaction of the end effector and the tool. Through our solution, we aim to achieve a 6D generalization of SEAs that can be useful for spatial tool use.  

Our characterization of spatial series-elastic actuators would not be complete unless we can measure the deformation of the spatial compliance in real time, and use the feedback for force control. In order to achieve this, we leverage recent advances in visuotactile sensing that measure 6D deformation using vision \cite{gelsight,gelslim,bubblegrippers,bigbubble}. In contrast to many works that utilize deep learning to directly process data from visuotactile sensors in an end-to-end manner, we propose to measure the pose of the grasped tool relative to the end effector, abstracting visuotactile sensing as a relative pose estimator.

Our proposed framework of Series Elastic End Effectors in 6D (SEED) consists of three elements: a manipulator capable of accurate position control, a 6D spatially compliant stiffness element, and visuotactile sensing that measures the deformation of the spatial compliance. With these three elements, we show that we can achieve spatial force control of tools with closed-loop feedback from visuotactile sensing. 

\section{Literature Review}
\subsection{Tool-Use in Manipulation}
Tool-use has long been one of the hallmarks of intelligence \cite{animaltooluse}, as well as a practical problem to solve for robotic applications. As such, many existing works \cite{toussainttooluse,holladay} center around how to give robots the ability to use tools. However, only a few works attempt to perform explicit force control with a tool that has not been rigidly attached to the robot, but rather, must be grabbed before it can be used.

Most existing works in this setting focus on planning, where the grabbed tool must be used to manipulate the pose of another object \cite{toussainttooluse,holladay,pushandpull}. Such plans can be very useful in reaching confined spaces \cite{pushandpull} or beyond the workspace of the manipulator \cite{toussainttooluse}. However, as the focus of these works lie more in planning, tasks that require force exchange among static objects, such as using a squeegee, surface grinding, wiping a table, or using torque drivers, are often not considered.

On the other extreme, classical works in robot force control excel in forceful manipulation with rigidly attached tools. Strategies such as impedance control \cite{hoganimpedancecontrol} and hybrid force-velocity control \cite{hybridpositionforce} have been extensively tested and applied on problems that require force exchange between the robot and the environment \cite{grinder,albuschaffer2,forcecontrol}. However, customized tool changers are quite limited in terms of versatility in the wild. 

Finally, works that attempt to explicitly apply forces with the grasped tool \cite{toussaintforce} often run into hardware limitations, as typical parallel jaw grippers with rigid, flat fingers are unable to resist forces and provide compliance in certain directions due to their relatively small contact patch.

\subsection{Manipulation using Visuotactile Sensing}

Visuotactile sensors \cite{gelsight,gelslim,bubblegrippers} consist of a deformable membrane which interacts with objects, and a camera (color, depth or both) under the membrane to measure its deformation during interactions. As the measurements from visuotactile sensors are images, some works have leveraged deep learning approaches to learn the dynamics \cite{swingbot}, or directly learn a map from the input image or optical flow to the policy \cite{visionandtouch,tactilerl}. While such approaches can be effective, we first focus on interpretable abstractions in this work that are more conducive for understanding, and may provide more inductive bias \cite{inductivebias} for designing deep models in the future.

Other works have taken a more model-based route. In \cite{tactiledexterity}, visuotactile sensors are used to track geometric features of the objects such as lines and points. These features are utilized to track the pose of the object and the contact state, which is then used for feedback control. Similarly, \cite{cablemanipulation} tracks the state of a deformable cable by estimating the contact patch ellipse, and fits a linear dynamics model which is stabilized by LQR. Although we use a similar model-based approach, our work is unique in that we generalize the estimator spatially, then explicitly do force control.

\subsection{Tactile Force and Pose Estimation}

Many of the existing works in tactile pose/force estimation attempt to deal with dense measurements. In \cite{contactpatchposeestimation}, ICP is used from dense depth information in order to estimate the poses of the object. Similarly, \cite{tactileposeestimation} uses geometric contact rendering which is then compared with the dense tactile image. While such dense information is useful for classification \cite{gelsight}, it is unclear if such dense information is necessary for control.

On the other hand, \cite{tactiledexterity} estimates simpler features such as points and edges, and \cite{cablemanipulation} estimates ellipsoidal contact patches that are sufficient for achieving the task. We use similar representation to \cite{cablemanipulation}, but estimate the patch in 3D instead of localizing on the plane. While such approaches are efficient to implement and is more relevant to the task, we note that they lack geometric generalizability compared to dense information.

\section{Preliminary: 1D Series Elastic Actuator}
In this section, we briefly review the concept of 1D SEAs, their proposed benefits, and the corresponding control strategies. Although the section will entirely be a review of previous work on SEAs, the ideas presented here will have direct correspondences with our generalization.

\subsection{1D Series Elasticity}
Closed-loop force control often requires a motor, a gearbox, and a force sensor in series. Typically, a relatively stiff sensor based on strain gauges is used. However, this force-feedback setting can result in instability due to high contact stiffness \cite{sea}, as well as non-collocation of sensors and actuators \cite{hoganimpedancecontrol,flexiblevehicles,macromicromanipulator}. This prevents the use of high-gains that are necessary to overcome undesired effects of the gearbox.

SEAs, initially proposed in \cite{sea}, can be understood as a special case where the sensor stiffness is very low. Under this setting, force-feedback enjoys better stability properties at the expense of controller bandwidth, as the spring acts like a mechanical low-pass filter \cite{stableactuator,hoganimpedancecontrol}. For many household tool-use tasks such as wiping with a squeegee, the loss of control bandwidth does not pose a big problem, as such tasks are usually quasistatic.  Thus, one may use high-gain position control to overcome unwanted effects of the gearbox, while still maintaining stability of the system and achieving greater force accuracy \cite{sea}.

\subsection{Force Control of Series Elastic Actuators}
We present a simple version of force control with series elastic actuators. In force control, the user supplies a desired force $f_d$, which can be turned into desired relative position using the sensor stiffness $k$. Then, a high-gain position controller can be used to achieve this relative position. The detailed procedure is described in Algorithm \ref{alg:1dforcecontrol}. In practice, frequency-domain analysis can be done to carefully choose gains that stabilize the closed-loop system \cite{seathesis}.
\begin{algorithm}[thpb]
\textbf{Given}: Desired force $\small ~^W\!f_d$, sensor stiffness $\small k$\;
Convert desired force to desired deformation $\small ~^T\!x^C_d=f_d/k$ \;
Convert desired spring deformation to desired position $\small ~^W\!x^T_d=~^W\!x^T+~^T\!x^C-~^T\!x^C_d$\;
Use position control to regulate to desired position $\small u=-k_p(~^W\! x^T-~^W\!x^T_d)-k_d ~^W\!\dot{x}^T$
 \caption{Force Control with SEAs}
 \label{alg:1dforcecontrol}
\end{algorithm}

\subsection{Multi-DOF SEAs for Tool-Use}\label{sec:2c}
How can we utilize the benefits of SEAs to multiple degrees-of-freedom? One straightforward answer might be to connect SEAs serially at the joint-level \cite{albuschaffer2}. However, achieving accurate end-effector position and force tracking using joint-level SEAs requires fast and accurate joint-level torque sensing, which is not available on many position-controlled robots. Instead, we offer an alternative generalization of SEAs that concentrate the 6D elasticity into the end effector, while allowing the robot to remain stiff. Our generalization involves the following three components:
\begin{enumerate}
    \item A 6D deformable element capable of being stiff in multiple directions simultaneously.
    \item A mechanism to sense the spatial deformation of the above element.
    \item A manipulator capable of controlling spatial pose of the deformable element.
\end{enumerate}

\section{SEED: Series Elastic end effector in 6D}
In this section, we present Series Elastic End Effectors in 6D (SEED), a spatial generalization of 1D SEAs that satisfies the three requirements in Sec.\ref{sec:2c} by using a soft deformable membrane, visuotactile sensing to sense the spatial deformation of the membrane, and a position-controlled manipulator to control the pose of the membrane base.
\subsection{Defining 6D Series Elasticity}
One of the challenges of generalizing the 1D SEA using a spatially compliant element comes from defining an appropriate notion of spatial stiffness \cite{cartesianmatrix}, especially for large rotations (rotations up to 30 degrees are common in our experiments). Rotational stiffness has been traditionally defined on the roll-pitch-yaw and axis-angle parameterization of rotations \cite{spatialimpedanceaxisangle, natale}, which can be made to work for large rotations. 

In this work we have chosen the \textit{bushing model}, which was initially proposed as a coordinate-free parameterization of a bushing element in Drake \cite{drake}. The bushing model also works for large rotations, and can be interpreted more intuitively due to its correspondence to a spring-loaded gimbal (Fig. \ref{fig:banner}). Based on the bushing model, we will develop a generalized stiffness map that relates the relative pose between two frames to a spatial force.


\subsection{Frame Definition} We give the definition of the frames here in order to better ground our notion of 6D series elasticity to the setting of a soft and tactile hand grabbing a tool. At the moment of grasp between the soft hand and the tool, two frames are initialized: $T$, which is rigidly attached to the gripper at a pre-defined nominal location (e.g. the center of the gripper), and $C$, which is rigidly attached to the tool and initialized to be coincident with $T$ (i.e. identity relative transform). 

\subsection{The Generalized Stiffness Map} Given the definition of these frames, our goal is to characterize the relation between the relative pose of $C$ with respect to $T$ (denoted as $\tensor[^T]{\mathbf{X}}{^C}\in\text{SE(3)}$) and the spatial force (written in frame $T$) applied on $C$, which we denote as $\tensor[^T]{F}{}\in\mathbb{R}^6$. We abstractly denote this as a \textit{generalized stiffness map} $\mathcal{K}:\text{SE}(3)\rightarrow \mathbb{R}^6$ such that Eq.\ref{eq:stiffnessmap} holds:
\begin{equation}
    \tensor[^T]{F}{} =\mathcal{K}(\tensor[^T]{\mathbf{X}}{^C}).
    \label{eq:stiffnessmap}
\end{equation}

We expect $\mathcal{K}$ to be a generalized notion of stiffness with smoothness and monotonicity properties under the following assumption of no-slip.

\begin{assumption}
    \normalfont No slip occurs between the contact patch of the gripper and the tool, such that $\mathcal{K}$ smoothly maps relative transform to spatial force.
\end{assumption}

We now concretely describe the bushing model $\mathcal{K}$. We denote ${}^T\Theta^C=[r,p,w]^{\intercal}$ as the roll-pitch-yaw parametrization  (which lives on a $\textit{gimbal}$) of $\tensor[^T]{\mathbf{R}}{^C}\in\text{SO}(3)$, and $\tensor[^T]{x}{^C}\in\mathbb{R}^3$ to be the position component of $\tensor[^T]{\mathbf{X}}{^C}$. Similarly, the spatial force $\tensor[^T]{F}{}$ is divided into torques $\tensor[^T]{\tau}{}\in\mathbb{R}^3$ and forces $\tensor[^T]{f}{}\in\mathbb{R}^3$. Then, the bushing model gives spatial force for a pose using the following relation:

\begin{equation}
    \small \begin{pmatrix} \tensor[^T]{\tau}{} \\ \tensor[^T]{f}{} \end{pmatrix} = \begin{pmatrix} \mathbf{N}^{\intercal}(\tensor[^T]{\Theta}{^C})\mathbf{K}_\tau \tensor[^T]{\Theta}{^C}  \\
    \mathbf{K}_f \tensor[^T]{x}{^C} \end{pmatrix}.
    \label{eq:bushingstiffness}
\end{equation}
where $\mathbf{K}_\tau\in\mathbb{R}^{3\times 3}$ is the \textit{gimbal stiffness matrix}, and $\mathbf{K}_f\in\mathbb{R}^{3\times 3}$ is the standard \textit{translational stiffness matrix}. $\mathbf{N}(\tensor[^T]{x}{^C})\in\mathbb{R}^{3\times 3}$ is the coordinate transformation matrix necessary to convert gimbal torques to spatial torques, and is given by

\begin{equation}
    \small \mathbf{N}(\tensor[^T]{\Theta}{^C})=\begin{pmatrix} \cos w \sec p & \sin w \sec p & 0 \\
                    -\sin w & \cos w & 0 \\
                    \cos w \tan p & \sin w \tan p & 1 \end{pmatrix}.
\end{equation}
The necessity of the matrix becomes apparent by visualizing $\mathbf{K}_\tau \tensor[^T]{\Theta}{^C}$ as torques that are being exerted on each axis of the gimbal, while $\tensor[^T]{\tau}{}$ is defined spatially in $\mathbb{R}^3$. We obtain the matrix by equating the power on a spatial representation $\tensor[^T]{\tau}{}\cdot \tensor[^T]{\omega}{^C}$ to power on the gimbal representation $\mathbf{K}_\tau \tensor[^T]{\Theta}{^C}\cdot \tensor[^T]{\dot{\Theta}}{^C}$, and using the standard conversion between angular velocities and gimbal rates. Throughout our work, we make the following assumption on the structure of $\mathbf{K}_\tau$ and $\mathbf{K}_f$.

\begin{assumption}
    \normalfont The gimbal stiffness matrix $\mathbf{K}_\tau$ and the translational stiffness matrix $\mathbf{K}_f$ are positive definite diagonal matrices. 
\end{assumption}

Under Assumption 2, we present the following theorem which gives a more rigorous notion of smoothness mentioned in Assumption 1. 

\begin{theorem}
\normalfont The bushing model stiffness map $\mathcal{K}:\text{SE}(3)\rightarrow\mathbb{R}^6$ is a diffeomorphism under Assumption 2 everywhere for $|p|\leq\pi/2$.
\label{thm:diffeomorphism}
\end{theorem}

\begin{proof}
 Since there is no coupling between the orientation and translational maps, it suffices to separately show that each are diffeomorphisms. The translational map given by $^Tf = \mathbf{K}_f \tensor[^T]{x}{^C}$ is trivially a diffeomorphism under Assumption 2. We use the Inverse Function Theorem to prove the inverse differentiability of the orientation map. The determinant of the Jacobian for the orientation map is given by 
 \begin{equation}
     \small \det \mathbf{J}=\frac{k_rk_pk_y}{\cos p},
 \end{equation}
 where $k_r,k_p,k_y$ are the diagonal elements of $\mathbf{K}_\tau$, which is well defined everywhere in $|p|< \pi/2$. We complete the proof by noting that the orientation map is bijective, and show it by providing a well-defined nonlinear inverse $\tensor[^T]{\Theta}{^C}=\mathcal{K}^{-1}(^T\tau)$:
    \begin{equation}
        \small
        \begin{pmatrix} r \\ p \\ y \end{pmatrix} = 
        \begin{pmatrix}
            \frac{1}{k_r}\big((\tau_x\cos y + \tau_y\sin y)\big)\cos p - \tau_z \sin p \\ 
            \frac{1}{k_p}\big(\tau_y \cos y - \tau_x \sin y\big) \\
            \frac{1}{k_y} \tau_z \\
        \end{pmatrix}.
        \label{eq:stiffnessinverse}
    \end{equation}
 Note that the equation is written in semi-implicit form to save space: one can easily make it explicit by substituting values starting from the bottom row. 
\end{proof}

Theorem \ref{thm:diffeomorphism} tells us that our model of $\mathcal{K}$ follows desirable properties that can smoothly map back and forth between relative pose deformation and spatial force, which we can effectively use in order to do force and impedance control in a manner akin to SEAs. We also note that in hardware, we expect $|p|$ to be confined to $\pi/3$ at most before Assumption 1 is broken and slip occurs. 

\section{Force control with SEED}

\begin{figure*}[thpb]
	\centering\includegraphics[width = 1.0\textwidth]{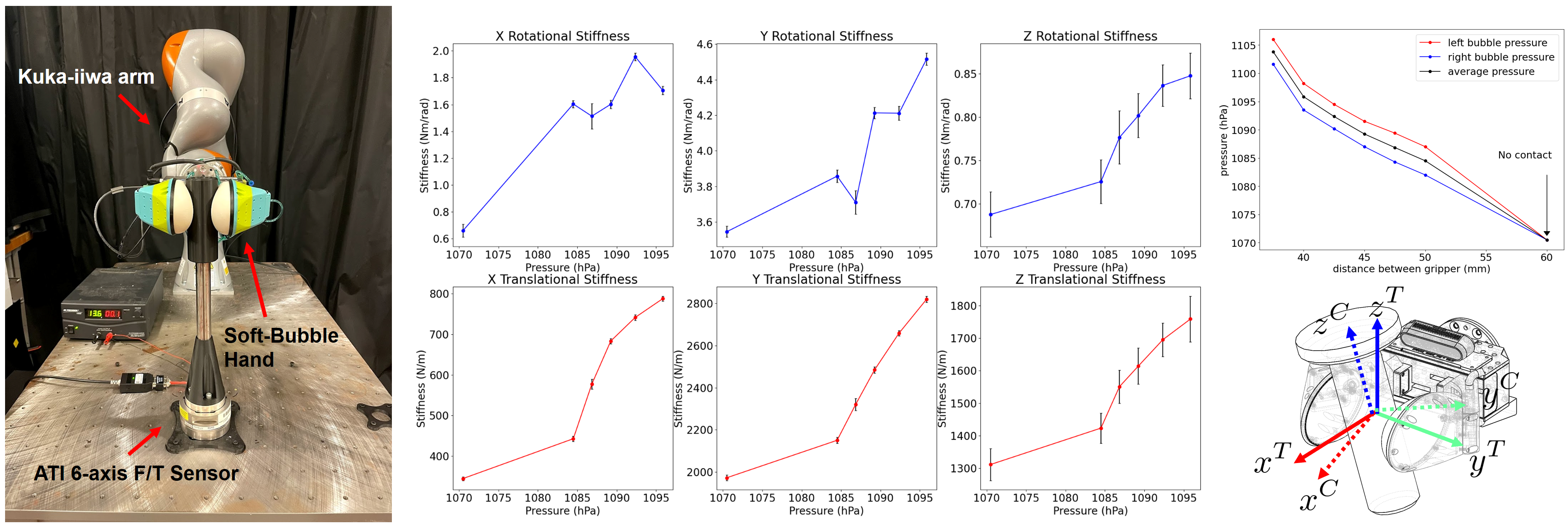}
    \caption{Left: System Identification Setup with the Soft-Bubble gripper \cite{bubblegrippers} and a 6-axis force/torque sensor. Center: Results of identified stiffness values for different axis, per change in pressure. Right Top: Change in pressure as a function of gripper distance command. Right Bottom: Frame definition.}
	\label{fig:sysid}
\end{figure*}

Now we present our main algorithms for doing control with SEED, which follows the general philosophy of controllers using SEAs: \textit{a force control problem is turned to a position control problem} \cite{sea}. Thus, we assume access to a manipulator that can achieve reliable position commands with high gains and rates (akin to how SEAs can use high gains to overcome gearbox effects quickly and achieve accurate positions), which describes most position controlled manipulators with high mechanical repeatability.

\subsection{Problem Setup - Feedback and Action} To setup the control problem, we note that the position controlled manipulator can command end-effector pose $u(t)=\tensor[^W]{\mathbf{X}}{^T}(t)\in \text{SE}(3)$ with high rates using direct inverse kinematics or integration of differential inverse kinematics. Our feedback signal will come from the estimation of relative pose  $y(t)=\tensor[^T]{\mathbf{X}}{^C}(t)$, which is measured by the visuotactile method given in Sec.\ref{sec:perception}. Then, the goal is to find a policy $u(t)=\pi(y(t))$ that achieves some desired specification of the user. 

Throughout the section, we will assume we have some estimate of parameters $\hat{\mathbf{K}}_\tau$ and $\hat{\mathbf{K}}_f$ that define the generalized stiffness map, and denote the estimated map as $\hat{\mathcal{K}}(\hat{\mathbf{K}}_\tau,\hat{\mathbf{K}}_f)$.

\subsection{6D Force Control} 
In force control, the user specifies some desired spatial force $\tensor[^W]{F}{_d}$, described in the world frame. SEED achieves this specified spatial force by converting it to some desired relative transform with the estimated generalized stiffness map $\hat{\mathcal{K}}$. Then, a position command is sent to the manipulator to achieve this relative pose. We describe the detailed process in Algorithm \ref{alg:forcecontrol}. 

\begin{algorithm}[thpb]
\textbf{Given}: Desired wrench $\tensor[^W]{F}{_d}$\;
\textbf{Given}: Estimated SEED stiffness $\hat{\mathcal{K}}(\hat{\mathbf{K}}_\tau,\hat{\mathbf{K}}_f)$\;
\While{$t < T$}{
  Convert $\tensor[^W]{F}{_d}$ to $\tensor[^T]{F}{_d}(t)$ using adjoint transform with current position $\tensor[^W]{\mathbf{X}}{^T}(t)$ \;
  Using the estimated SEED stiffness $\hat{\mathcal{K}}$, convert desired wrench into desired relative pose by $\tensor[^T]{\mathbf{X}}{^C_d}(t)=\hat{\mathcal{K}}^{-1}(\tensor[^T]{F}{_d}(t))$\; 
  Convert desired relative pose into desired end effector pose by $\tensor[^W]{\mathbf{X}}{^T_d}(t)=\tensor[^W]{\mathbf{X}}{^T}(t)\cdot \tensor[^T]{\mathbf{X}}{^C}(t) \cdot \tensor[^C]{\mathbf{X}}{^T_d}(t)$\;
  Send position command $\tensor[^W]{\mathbf{X}}{^T_d}(t)$ to the position controller.
 }
 \caption{Force Control with SEED}
 \label{alg:forcecontrol}
\end{algorithm}

The expression for the orientation part of $\mathcal{K}^{-1}$ has been given in Eq.\ref{eq:stiffnessinverse}, while inverting the position simply requires $\mathbf{K}^{-1}_f$. We note that, like most force control strategies, the controller will not be well-behaved if there is no contact with an external environment. In particular, while position-only force control can move until contact and maintain some desired force, the orientation torque controller must keep rotating until contact, which likely runs into workspace limitations of the manipulator quickly and makes the controller impractical to use in free-space. 

\subsection{6D Hybrid Force/Pose Control}
In many tasks involving tools, the goal is to simultaneously control force and torque in certain directions, while controlling position and orientation in other directions. We naturally extend spatial force control with SEED to this setting by defining a partial inverse of the impedance map $\mathcal{K}$ that attempts to construct the spatial deformation $\tensor[^T]{\mathbf{X}}{^C_d}$ from a subset of specified forces. 

\subsubsection{Hybrid Force/Position Control}

Given a task-relevant decomposition matrix $\mathbf{P}$ which selects a subspace for desired position $x_d\in\mathbb{R}^3$ and desired force $f_d\in\mathbb{R}^3$, we can compute the position $x\in\mathbb{R}^3$ that achieves the specified positions and forces:
\begin{equation}
    \small x = \mathbf{P}x_d + \mathbf{P}^\perp\mathbf{K}_f^{-1} f_d,
\end{equation}
where $\mathbf{P}^\perp$ is the matrix that represents the orthogonal complement of $\mathbf{P}$.  

\subsubsection{Hybrid Torque/Orientation Control}
Unlike force / position control, coordinate transform in rotational space does not happen in a linear manner. Thus, defining hybrid torque/orientation control for an arbitrary task-relevant coordinate representation is significantly more difficult. To deal with this problem, we make the following assumption:

\begin{assumption}
    \normalfont The decomposition of specified orientation and torques happen in the frame of $T$. 
\end{assumption}

Such an assumption is not too restrictive under a large class of tools, as most tools require decomposition of torques and angles in a manner consistent with its natural task-relevant coordinate frame. Under such assumption, we can define partial maps from a subset of desired torques to the full orientation as follows:
\begin{enumerate}
    \item \textbf{2 torques, 1 angle}: The following angles $(r,p,y)$ achieve the given two desired torques $\tau_x^d,\tau_y^d$ and one desired angle $y_d$, given the stiffness map $\mathcal{K}$:
    \begin{equation}
    \small
    \begin{pmatrix} r \\ p \\ y \end{pmatrix} = \begin{pmatrix}
    \frac{1}{k_r}((\tau_x^d \cos y_d + \tau_y^d\sin y_d)\cos p - k_y y_d \sin p) \\
    \frac{1}{k_p}(\tau_y^d \cos y_d - \tau_x^d \sin y_d) \\
    y_d
    \end{pmatrix}.
\end{equation}

    \item \textbf{1 torque, 2 angles}: The following angles $(r,p,y)$ achieves the given torque $\tau_z^d$ and the two desired angles $r_d,p_d$, given the stiffness map $\mathcal{K}$:
    \begin{equation}
    \small
    r = r_d \qquad p = p_d \qquad y = \tau_z / k_y. 
    \end{equation}
\end{enumerate}

After recovering the full pose from a subset of desired forces and torques, we use position control to command this pose, as done in Alg. \ref{alg:forcecontrol}.

\section{System Identification}
In order to apply our framework, we need to do identification on the parameters of the stiffness map $(\mathbf{K}_\tau,\mathbf{K}_f)$, which is consisted of $6$ stiffness parameters. The stiffness parameters can be identified by measuring the static sensitivity of wrench with respect to pose. We achieve this by having a dexterous manipulator grab a 6-axis force/torque sensor and perturbing the pose to observe responses in wrench.

In addition, to see if squeezing or pressurizing the gripper changes the stiffness parameters of the hand, we use the pressure sensor on board the soft-bubble hand \cite{bubblegrippers} to characterize how the gripper distance affects the pressure, and in turn, how the pressure affects the identified stiffness values. The results of our experiments are presented in Fig.\ref{fig:sysid}. 

Along with the quantitative values of stiffness, we summarize our findings from the identification process:

\begin{enumerate}\itemsep -0.1em
    \item The dependence on internal pressure of the hand with respect to the gripper distance is linear.
    \item For $x$ direction torque and all the forces, higher pressure near-linearly corresponds to higher stiffness. The identified stiffness values also have low standard deviation.
    \item For $y,z$ direction torque, the measurement is relatively unreliable and the identified stiffness values are subject to large standard deviations. In addition, higher pressure does not seem to lead to higher stiffness values along these directions. 
\end{enumerate}

The results of system identification, combined with the monotinicity of the stiffness map, leads to a very natural interpretation: if stiffer behavior is desired while controlling the tool, the hardware gives us the means to control the stiffness by grabbing the tool more firmly or by more pressurization.

\section{Tactile Relative Pose Estimation}
\label{sec:perception}

In principle, our framework of control can work well with any tactile end effector that is compliant enough, and an estimation algorithm that produces a well-behaved estimate of the relative pose $\tensor[^T]{\mathbf{X}}{^C}$. In our work, we show an example of such a relative pose estimator by utilizing the PicoFlexx IR-Depth camera mounted within the bubble grippers \cite{bubblegrippers}. 

\subsection{Contact Patch Estimation}

We estimate the position of the contact patch using a simple background subtraction algorithm. Denote $\mathbf{D}_k\in\mathbb{R}^{H\times W}$ as depth image at time $k$. Then, we simply compare $\mathbf{D}_k$ to the initial depth image $\mathbf{D}_0$, taken when the bubble is not in contact. After performing a thresholding operation to obtain the difference, we perform a morphological transformation using an elliptical kernel to obtain a binary mask $\mathbf{M}_k\in\{0,1\}^{H\times W}$. Finally, we use the calibration matrix $\mathbf{K}\in\mathbb{R}^{3\times 3}$ to transform the masked depth image $\mathbf{M}_k \odot \mathbf{D}_k$ into a set of points $\{{}^L p_i^{P_L}\}$, where $\odot$ denotes element-wise multiplication, $L$ denotes the left camera frame, and $P_L$ denotes the left contact patch. Finally, we take a mean ${}^L p^{P_L} = \frac{1}{N}\sum^N_{i=1} {}^L p_i^{P_L}\in\mathbb{R}^3$ to obtain the 3D coordinate of the contact patch, and repeat this process for the right camera.

\subsection{Frame Estimation from Contact Patches}
Given the location of the contact patch on the left bubble expressed in the gripper frame ${}^G p^{P_L}$ and ${}^G p^{P_R}$, we average the positions of the two patches to obtain the position of the contact frame:
\begin{equation}
    {}^G p^C = ({}^Gp^{P_L} + {}^G p^{P_R})/2.
\end{equation}

To compute the rotation $^{G}\mathbf{R}^C$, we introduce an intermediate frame $C'$ such that the $y$ axis of $C'$ is aligned with ${}^G p^{P_R}-{}^G p^{P_L}$, and the $z$ component of the $x$ axis is zero. Denote $u,v,w$ as columns of ${}^G\mathbf{R}^{C'}$ (i.e. ${}^G\mathbf{R}^{C'}=[u|v|w]$), which individually represent the components of the unit vector that define ${}^G\mathbf{R}^{C'}$.  Then, we compute the columns using the following process:
\begin{enumerate}\itemsep -0.1em
    \item Set $v$ to be the normalization of ${}^G p^{P_R}-{}^G p^{P_L}$.
    \item Set $u_z=0$ to define the zero-pitch frame, and $u_x=1$
    \item Set Compute $u_y$ by using $u\cdot v=0$, and normalize $u$. 
    \item Set $w=u\times v$
\end{enumerate}

\begin{figure}[b]
	\centering\includegraphics[width = 0.5\textwidth]{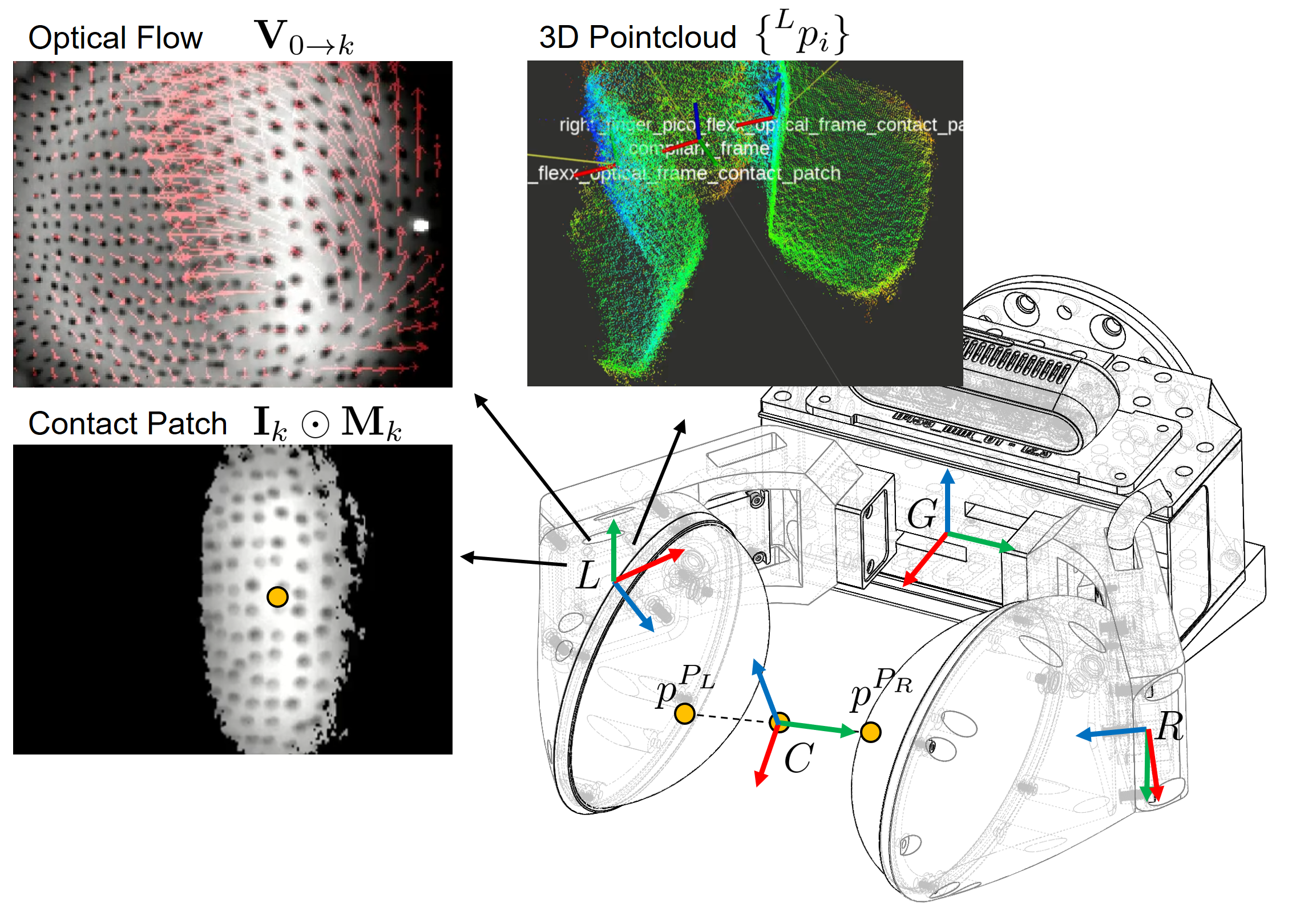}
    \caption{Frame Definition and example images for Relative Pose Estimation.}
	\label{fig:rpe}
\end{figure}

\subsection{Pitch Estimation with Optical Flow}

After computing ${}^{C'}\mathbf{R}^C$, we compute ${}^{C'}\mathbf{R}^{C}$, by computing the rotation along the $y$ axis of $C'$. We estimate this quantity by computing the optical flow of the IR image. We denote as $\mathbf{V}_{0\rightarrow k}$ as the Eulerian flow of $\mathbf{I}_k$ relative to $\mathbf{I}_0$. Then, we compute the curl of $\mathbf{V}_{0\rightarrow k}$:

\begin{equation}
   \small \theta = k\nabla_\times \mathbf{V}_{0\rightarrow k} = k\bigg(\frac{\partial}{\partial x} \mathbf{V}_{0\rightarrow k}^y -  \frac{\partial}{\partial y} \mathbf{V}_{0\rightarrow k}^x\bigg),
\end{equation}
where the superscript denotes the component of the vector field, and $k$ is some normalization constant we calibrate for. The gradients are computed using a Sobel filter with corresponding kernels.

\subsection{Validation Results}

In order to validate the performance of the proposed relative pose estimator, we use the same setup that was used for system identification (Fig.\ref{fig:sysid}). Through the forward kinematics of the manipulator, and the fact that ${}^W \mathbf{X}^C$ is a fixed transform for the system identification setup, we compare the measured values of ${}^G\mathbf{X}^C$ with the results of the relative pose estimator ${}^G\hat{\mathbf{X}}^C$. Our results, illustrated in Fig.\ref{fig:rpetracking}, show that tracking performance of relative pose is dependent on which axis is being tracked:

\begin{figure}[b]
	\centering\includegraphics[width = 0.5\textwidth]{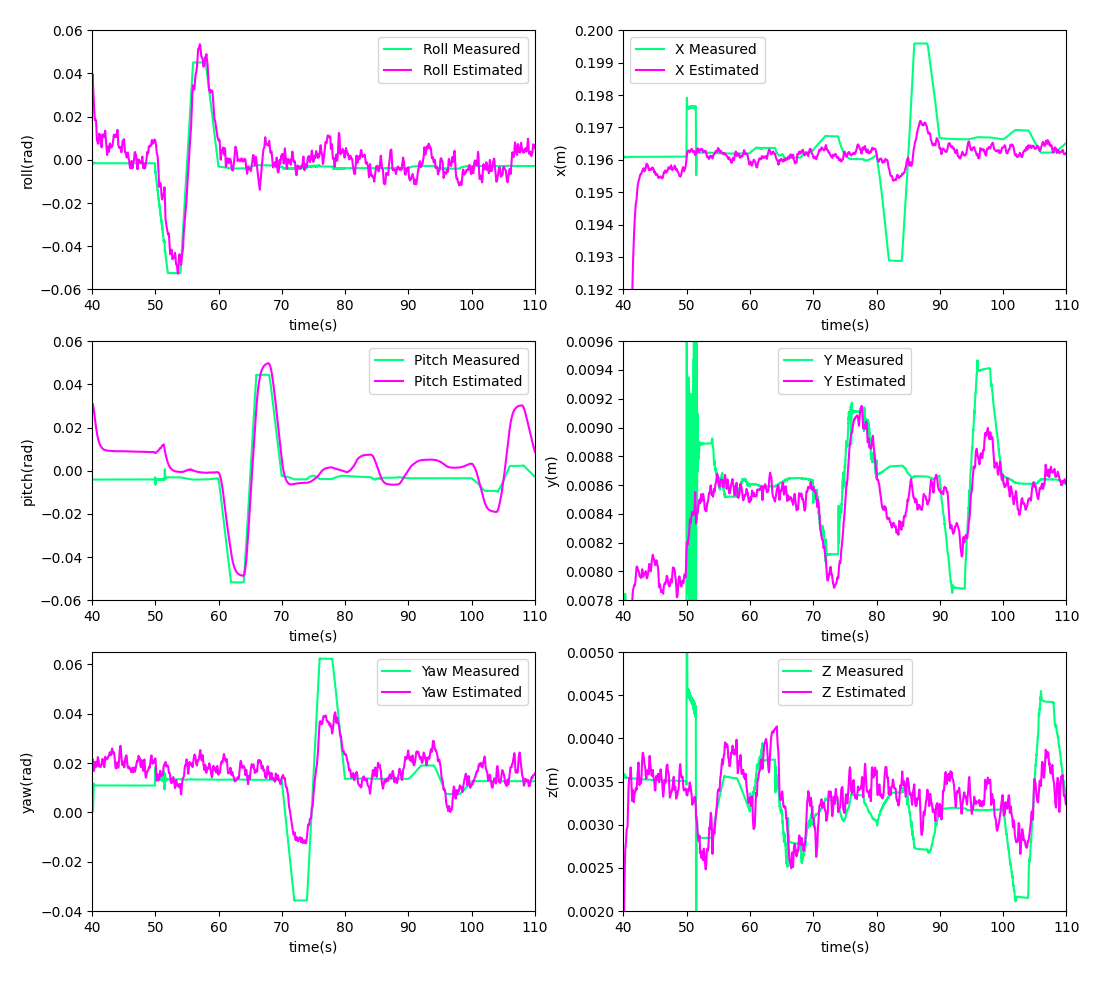}
    \caption{Tracking performance of the Relative Pose Estimation}
	\label{fig:rpetracking}
\end{figure}

\begin{enumerate}
    \item The $y$ position, which uses depth information from each camera, can be tracked reliably.
    \item On the other hand, the $x$ and $z$ position tracking is not very reliable due to the large contact patch caused by the cylindrical geometry of the tool. 
    \item The locations of contact patches on both sides give a very good estimate for roll angle. Optical flow is also successful in tracking pitch. 
    \item While yaw shows reasonable behavior, the estimate tends to underestimate the true yaw angle as the contact patch lags behind true rotation due to the softness of the membrane (i.e. perfect roll does not occur).
\end{enumerate}

\section{Experiment Methods \& Results}
\subsection{Simulation Methods \& Results}
To verify the performance of our proposed pipeline, we first set up a simulation in Drake \cite{drake}, where the compliance between the tool frame $T$ and the compliance frame $C$ is simulated using Drake's 6D compliance element \texttt{LinearBushingRollPitchYaw}. By assuming a perfect measurement of the relative pose ${}^T \mathbf{X}^C$, we aim to decouple the validity of the proposed controller with the accuracy of the tactile relative pose estimator. 
\subsubsection{The Squeegee Task}

The squeegee is a tool that requires regulation of spatial forces along some principle axis, while requiring regulation of position along others. We illustrate the frame definition in Fig.\ref{fig:squeegee}, and decompose the spatial forces and positions in the following directions in order to set a task specification of the hybrid force position controller:
\begin{enumerate}
    \item ${}^W x^C, {}^W y^C, {}^T \psi^C$ are used for position control in order to specify the trajectory of the tool from a table-top view.
    \item ${}^T \tau_y$ and ${}^W f_z$ are used to enforce the magnitude of pressure between the blade and the table.
    \item ${}^T \tau_x$ is used to enforce equal pressure distribution.
\end{enumerate}

As a baseline, we include an open-loop trajectory that is tuned such that the squeegee barely contacts the table, within the mechanical repeatability of the manipulator (0.1mm). In addition, we modify the controller for the case where the tool is rigidly fixed (\textit{welded}) to the end effector in order to simulate the performance of a custom tool changer. The resulting contact forces are inspected based on how much force is exerted ($F_z = \sum_i \lambda^i$), and how much the pressure distribution on the blade is balanced ($\tau_x = \sum_i r_i \times \lambda^i$). The resulting trajectory is shown in Fig.\ref{fig:contact_results}.

We show that compared to the case where the end effector is rigidly attached, the compliant hardware allows much better tracking of $\tau_x$, such that equal pressure is applied on both sides of the squeegee. We mainly attribute this behavior to the built-in compliance, as $\tau_x$ behaves well even in the open-loop setup. By commanding the desired force in closed-loop however, the 6D hybrid force-position controller adds the ability to exert desired amount of forces. Finally, we note that there exists offset in the tracking error due to the unaccounted weight of the tool.

\begin{figure}[thpb]
	\centering\includegraphics[width = 0.5\textwidth]{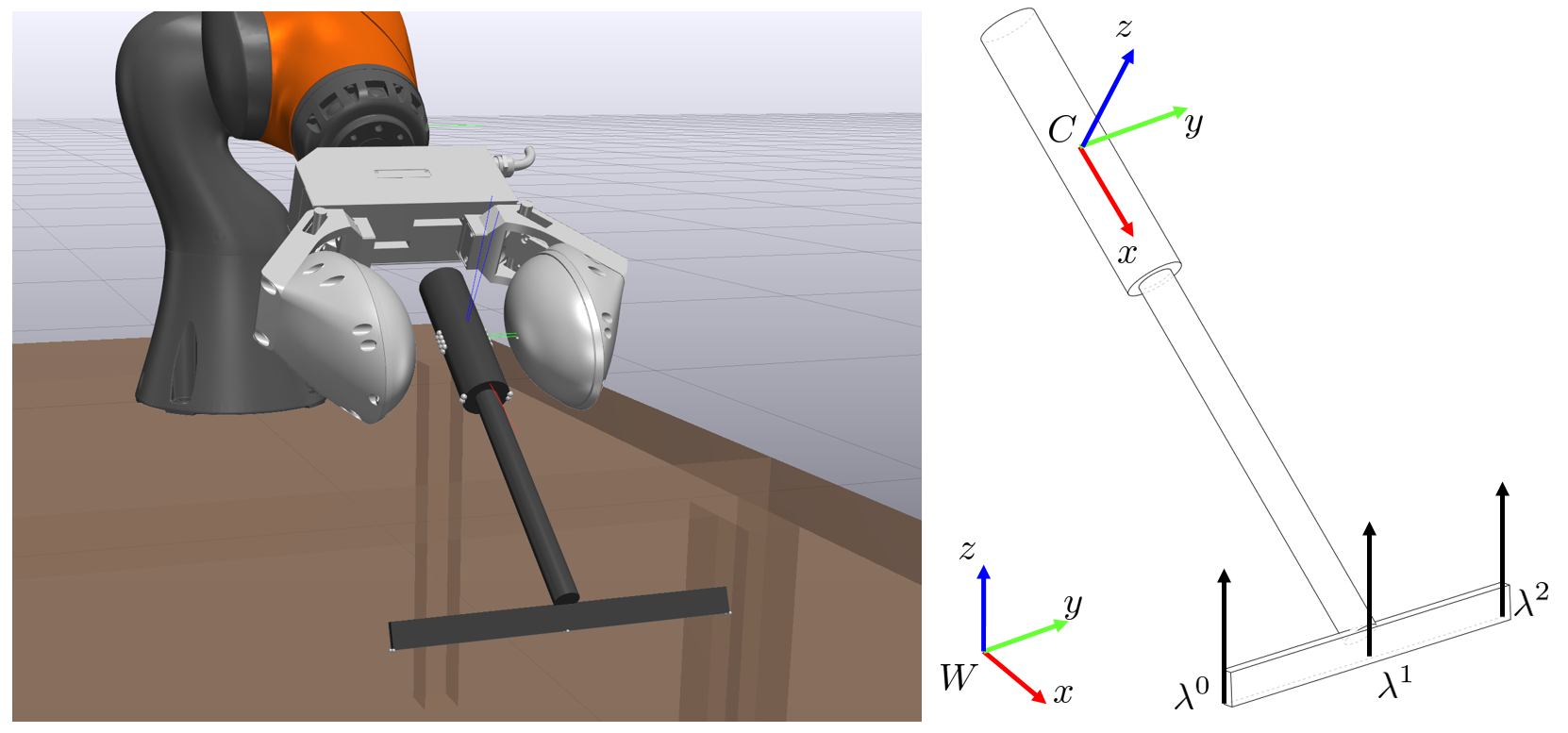}
    \caption{Left: Drake Simulation environment. Right: Frame and contact point definitions for the squeegee tool.}
	\label{fig:squeegee}
\end{figure}

\begin{figure}[thpb]
	\centering\includegraphics[width = 0.5\textwidth]{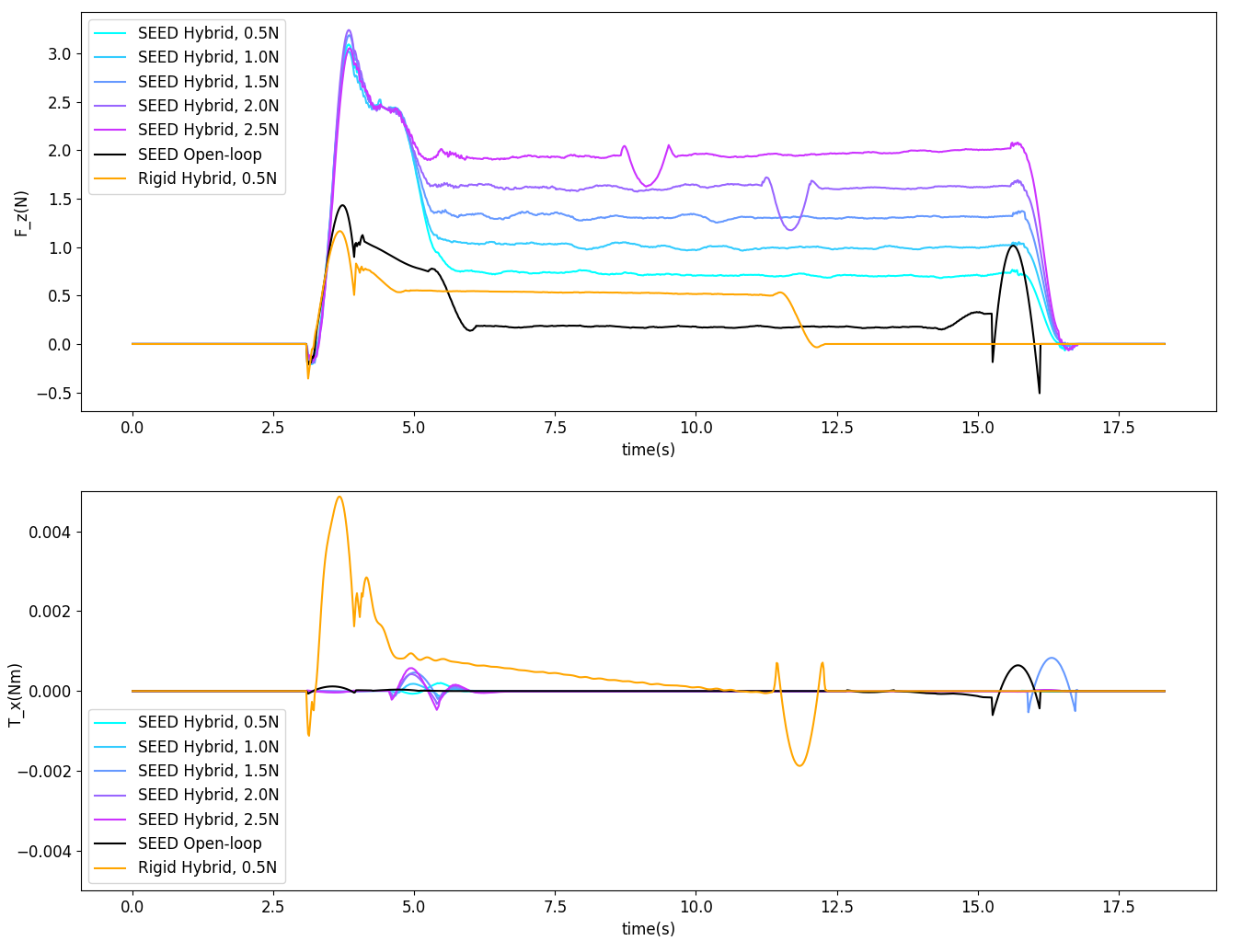}
    \caption{Controller results of tracking $f_z$ and $\tau_x$.}
	\label{fig:contact_results}
	\vskip -0.2 true in
\end{figure}
\subsection{Hardware Methods \& Results}

Though we have verified the behavior of the controller in simulation assuming perfect pose tracking, showing the controller on hardware requires coupling the pose estimator and the controller in all six axes. However, the estimator is unreliable in certain directions such as yaw or $x$-position, which can adversely destabilize closed-loop behavior.

In order to overcome these limitations of the estimator, we propose a simple yet effective strategy: we purposely align the axis that requires force tracking with the axis that our estimator performs well in. As most tasks require at most two or three components of force tracking, we show that it is possible to only estimate well a subset of the relative pose, and still achieve the underlying task.

\subsubsection{Pen Writing Task}
We first test the controller on a pen-writing task, where the robot is commanded to write some characters in the $xy$ plane, while some force is commanded in the $z$ direction. Our setup is illustrated in Fig.\ref{fig:hardware}.D. As the result of Fig.\ref{fig:hardware}.B demonstrates, our controller achieves good tracking performance of specified force, as observed by the differences in marker stroke width and darkness.

We also test our controller by writing letters in Fig.\ref{fig:hardware}.C. While we are successful in tracking the characters, the inherent softness of the hardware sacrifices the bandwidth of the position controller, and frictional interactions between the marker and the paper (e.g. caused by the Painleve effect) can compromise the tracking performance of SEED.

\subsubsection{Squeegee Task}
We test the proposed controller on a real-life task of using a squeegee to clean some liquid on top of a cutting board. The results of our hardware experiment are shown in Fig.\ref{fig:hardware}.A. While the open-loop baseline fails to exert much force on the board, the closed-loop controller is successful in pressing down firmly and clearing all the liquid. 

\begin{figure*}[t]
	\centering\includegraphics[width = 1.0\textwidth]{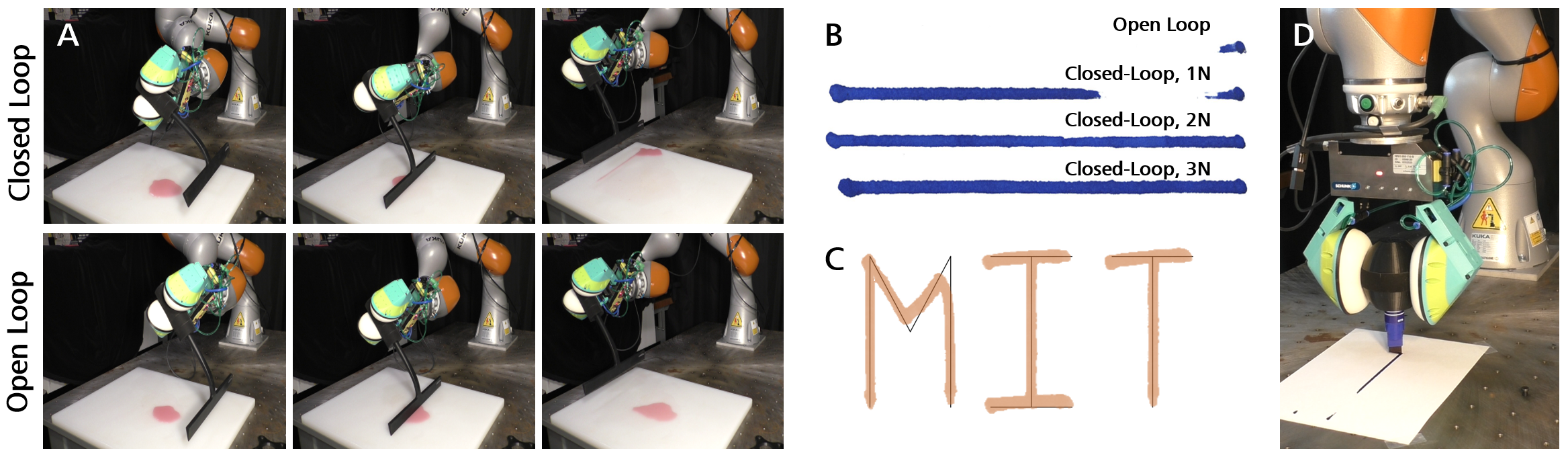}
    \caption{A. Deploying our controller on hardware for using the squeegee (top row), against the open-loop baseline (bottom row). B. Comparison of open-loop stroke against closed-loop strokes with different specified downward forces. C. Performance of character tracking. D. Pen writing setup.}
	\label{fig:hardware}
	\vskip -1.4em
\end{figure*}

\section{Conclusion} 
\label{sec:conclusion}

We have presented SEED, a control and hardware framework that combines the benefits of hardware compliance with visuotactile sensing. Throughout our work, we have demonstrated that we can measure the relative pose of a tool with respect to the gripper using visuotactile sensing. Combined with offline-identified parameters of our spatial stiffness model, we have shown that we can achieve closed-loop spatial force control that can be useful for tool-use. By our demonstration, we aim to alleviate some of the difficulties that rigid contacts and the associated non-smooth behavior bring in the setting of grasping and using tools in the wild.

\renewcommand{\bibfont}{\scriptsize}
\bibliographystyle{plainnat}
\bibliography{references}

\end{document}